\pdfoutput=1

\documentclass[11pt]{article}

\usepackage{amsmath,amsfonts,bm}

\def\eqref#1{equation~\ref{#1}}

\def\1{\bm{1}}

\def\vg{{\bm{g}}}
\def\vh{{\bm{h}}}

\def\vp{{\bm{p}}}

\def\vx{{\bm{x}}}

\DeclareMathAlphabet{\mathsfit}{\encodingdefault}{\sfdefault}{m}{sl}
\SetMathAlphabet{\mathsfit}{bold}{\encodingdefault}{\sfdefault}{bx}{n}

\usepackage[]{acl}

\usepackage{times}
\usepackage{latexsym}

\usepackage[T1]{fontenc}

\usepackage[utf8]{inputenc}

\usepackage{microtype}

\usepackage{algorithm}
\usepackage{algorithmic}
\usepackage{microtype}
\usepackage{amsmath}
\usepackage{multirow}
\usepackage{amsmath,array,graphicx,amssymb}
\usepackage[export]{adjustbox}
\usepackage{subcaption}
\usepackage{multirow}
\usepackage{mathtools} 
\usepackage{tabularx}
\usepackage{breakcites}
\usepackage{xcolor}
\usepackage{array}
\usepackage{makecell}
\usepackage{tabularx}
\usepackage{mathrsfs}
\usepackage{scalerel}
\usepackage{soul}
\usepackage{colortbl}
\definecolor{mygray}{gray}{.95}
\usepackage{soul}
\usepackage{booktabs}

\usepackage{graphicx}
\usepackage{caption}
\usepackage{algorithm}
\usepackage{algorithmic}
\usepackage{microtype}
\usepackage{amsmath}
\usepackage{multirow}
\usepackage{amsmath,array,graphicx,amssymb}
\usepackage{caption}

\usepackage{amsthm}
\newtheorem{theorem}{Theorem}

\newtheorem{definition}{Definition}

\usepackage{soul}
\usepackage{color}
\usepackage{xcolor}

\definecolor{apricot}{RGB}{255,240,234}
\definecolor{amber}{RGB}{214,81,30}

\definecolor{lightcyan}{RGB}{229,250,245}
\definecolor{teal}{RGB}{0,121,86}

\DeclareRobustCommand{\hlapricot}[1]{{\sethlcolor{apricot}\hl{#1}}}
\DeclareRobustCommand{\hllightcyan}[1]{{\sethlcolor{lightcyan}\hl{#1}}}
\usepackage{tikz}
\newcommand\dottedcircle{\tikz \draw [line cap=round, line width=0.25ex, dash pattern=on 0pt off 2pt] (0,0) circle [radius=0.75ex];}
\newcommand{\wyshi}[1]{\textcolor{black}{#1}}
\newcommand{\wyshiwyshi}[1]{\textcolor{black}{#1}}

\title{Selective Differential Privacy for Language Modeling}

\author{Weiyan Shi$^{1}$, Aiqi Cui$^{1}$, Evan Li$^{1}$, Ruoxi Jia$^2$ Zhou Yu$^1$\\
Columbia University$^1$,Virginia Tech$^2$ \\
\{ws2634, ac4788, el3078\}@columbia.edu, ruoxijia@vt.edu, , zy2461@columbia.edu}

\begin{document}
\maketitle
\begin{abstract}
With the increasing applications of language models, it has become crucial to protect these models from leaking private information. Previous work has attempted to tackle this challenge by training RNN-based language models with differential privacy guarantees.
However, applying classical differential privacy to language models leads to poor model performance as the underlying privacy notion is over-pessimistic and provides \emph{undifferentiated} protection for all tokens in the data. Given that the private information in natural language is sparse (for example, the bulk of an email might not carry personally identifiable information), we propose a new privacy notion, \emph{selective differential privacy}, to provide rigorous privacy guarantees on the sensitive portion of the data to improve model utility. To realize such a new notion, we develop a corresponding privacy mechanism, Selective-DPSGD, for RNN-based language models. Besides language modeling, we also apply the method to a more concrete application~--~dialog systems. Experiments on both language modeling and dialog system building show that the proposed privacy-preserving mechanism achieves better utilities while remaining safe under various privacy attacks compared to the baselines. The data and code are released to facilitate future research\footnote{\url{https://github.com/wyshi/lm_privacy}}.
\end{abstract}

\section{Introduction}
\begin{figure}
    \centering
    \includegraphics[scale=0.65]{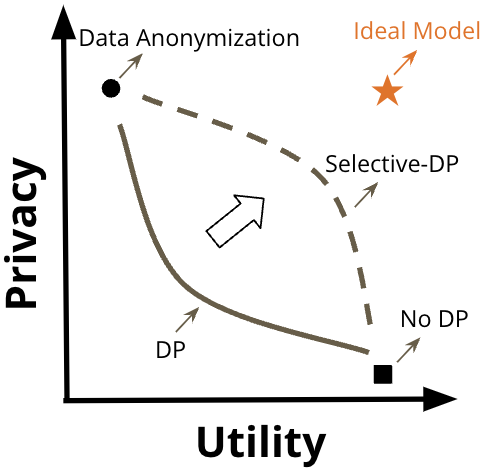}
    \caption{\wyshiwyshi{``Data anonymization'' and training with ``No DP'' \emph{cannot} provide knobs to adjust the privacy-utility trade-off. Selective-DP improves the privacy-utility trade-off of traditional DP, to get closer to the ideal model with both high privacy and high utility.}  }
    \label{fig:intro}
\end{figure}

Language models have been widely used in various kinds of applications, such as Google Smart Compose, Amazon Alexa, and so on. However, these models are often trained on highly sensitive data such as emails and chat logs, while having the tendency to memorize the training data unintentionally 
\citep{carlini2019secret, carlini2020extracting}.  Therefore, how to protect user privacy while preserving the model utility has become an increasingly important topic. 

Several methods have been developed to protect the data, such as data anonymization, $k$-anonymity, and differential privacy (DP). Among them, DP has become a dominant privacy notion as it provides formal and provable guarantees for people to understand the privacy-utility trade-off (Figure~\ref{fig:intro}). It works by carefully randomizing the algorithm so that the model does not rely too much on any single data point. 
However, traditional DP notion protects each data point as a whole regardless of the property of individual attributes inside the data 
\citep{mcmahan2018learning}. Training large models with this overly pessimistic privacy notion could lead to poor model performance or even a non-convergent training process~\cite{kerrigan2020differentially}. Our work is inspired by an important observation that in many scenarios including language modeling, 
private information is sparse,
and not all attributes need to be protected. For example, for the sentence ``My SSN is 123-45-6789'', only the last token with the actual SSN number needs to be protected. But if we protect the entire sentence, we may fail to learn the underlying language pattern well. 

To solve this problem, we propose a new DP notion, namely \emph{selective differential privacy} (S-DP), to provide focused protection for sensitive attributes \wyshi{in a training record} and improve model utility. \wyshi{
We follow the traditional DP setup, where each training record is contributed by a different user. The key difference is that we consider only partial dimensions of a training record as sensitive, which better abstracts the language data privacy problems}.  Whether a given dimension in a record is sensitive is specified by a user-defined policy function that encodes application-specific privacy regulations, \wyshiwyshi{which gives users the freedom to protect any type of sensitive information according to the use cases.} 
We also develop a corresponding privacy mechanism, Selective-DPSGD, for RNN-based language models under the new notion. 
Moreover, to process the variable-length sequences, we propose a batching method to group private and public tokens together and alternatively update them for a more efficient implementation of Selective-DPSGD. 


\wyshiwyshi{
One important concept to note  is that, as shown in Figure~\ref{fig:intro}, there is always a trade-off between privacy and utility:  data anonymization achieves high privacy guarantee at the cost of low utility on private tokens; models trained without DP have high utility but low privacy. However, for both data anonymization and ``no DP'', there is no way to tune the trade-off, while DP-related methods provide knobs to adjust the privacy-utility trade-off. This paper proposes S-DP to improve the trade-off of canonical DP to get closer to the ideal model with both high-utility and high-privacy. }

We evaluate S-DP on two tasks, 1) a language generation task and 2) a more concrete application of dialog systems. 
Besides reporting the model utility and theoretical privacy guarantees, we also empirically demonstrate their robustness under popular privacy attacks on language data~\cite{carlini2019secret,carlini2020extracting}.
The experiments suggest that training with Selective-DPSGD improves the model utility while remaining safe to the attacks.

Our contributions are as follows. First, we propose a \wyshi{\textbf{new}} selective differential privacy \textbf{notion} that ensures targeted privacy protection for sensitive attributes, and a corresponding mechanism to realize the new S-DP notion for RNN-based models.
Second, we propose a dialog dataset for future privacy  research.  
Next, we show both theoretically and practically that our models are safe to attacks with improved utilities on both the language generation task and the dialog system application.  \wyshiwyshi{Moreoever, we discuss the case of imperfect policy function, and compare S-DP with data anonymization to show that S-DP achieves better utilities when the policy function is imperfect. } We also show preliminary results on contextual policy functions and Transformer models. 
In the era of information explosion and large-scale pretraining, protecting data privacy becomes more and more important. With S-DP, we march one step closer towards more privacy-preserving and high-performing language models and hope to inspire more research in this direction in the NLP community. Moreover, despite our focus on language-related applications in this paper, the proposed S-DP notion could be useful for a much broader range of applications where only partial data require privacy protection.

\section{Related Work}

Language modeling is a key research problem in NLP. 
However, although language models are often trained on sensitive data such as emails,  most related studies focus on improving model without considering privacy, leaving the models vulnerable to attacks. For example, \citet{carlini2019secret} showed that it is possible to infer a secret in the training data by attacking  published language models. 
Therefore, it is of great importance to introduce privacy protection mechanisms to the NLP community and train the models in a much safer way.


Differential privacy (DP) \citep{dwork2014algorithmic} has been applied to various domains \cite{cortes2016differential, abowd2018us}. 
\citet{abadi2016deep} introduced DPSGD to train DP-protected deep-learning models. 
PATE~\cite{papernot2018scalable} leveraged knowledge distillation to train differentially private models. 
But DP-protected algorithms suffer from low utility, so
 the DP notion often needs to be adjusted according to the applications: \citet{ebadi2015differential} proposed a personalized DP notion to provide different levels of protection for different users; \citet{doudalis2017one} developed one-sided DP to protect sensitive users only. We also propose a new Selective-DP notion to protect only the sensitive attributes in a record to improve  utility.

Recently, DP has also been applied to NLP tasks \cite{ fernandes2019generalised, xu2020differentially,yue2021differential, hathurusinghe2021privacy, sasada2021differentially}.  For example, \citet{mcmahan2018learning} proposed DP-FedAvg and DP-FedSGD to train RNN language models with user-level privacy guarantees.  
\citet{adelani2020privacy} developed a probabilistic text de-identification algorithm with formal privacy guarantees
.
Different from existing work that directly applied DPSGD and provided undifferentiated protection for all training examples, we propose a new privacy notion and a corresponding mechanism to protect the sensitive portion of the data 
in centralized learning. Such a new notion can be easily adapted to federated learning settings as well. 

There is also a line of work that attempts to improve differentially private deep learning performance via modifying training process~\citep{wang2021dplis,wang2019dplssgd,thakkar2019differentially,lee2020differentially} or model architecture~\citep{TemperedSigmoid}. Our work is complementary to this line of work as we propose a new privacy notion. Particularly, our work can be combined with the aforementioned methods to further improve model utility.


\section{Backgrounds}

We will introduce language modeling and differential privacy as preliminaries in this section. 

\noindent\textbf{Language Modeling.} Consider a text sequence that consists of multiple tokens, i.e., $\vx = (x_1, x_2, \dots, x_{n})$, where $x_i$ is the $i$-th token. The goal of language modeling is to learn the probability of the sequence $p(\vx)$, which can be factorized with the chain rule as in Equation~(\ref{eq:lm}). Given a corpus $D = \{\vx^1, \dots, \vx^{|D|}\}$, we train a neural network (e.g., RNN) parameterized by $\theta$ to learn $p(\vx)$ by minimizing the negative log-likelihood over $D$ with the loss function in Equation~(\ref{eq:lm1}).
\begin{align}
\label{eq:lm}
    &p(\vx) = \prod_{i=1}^np(x_i| \vx_{<i}),\\
\label{eq:lm1}
    &\mathcal{L}(D) = -\sum_{t=1}^{|D|} \sum_{i=1}^{n_t}\log p_{\theta}(x_i^t|\vx_{<i}^t)
\end{align}


\noindent\textbf{Differential Privacy (DP)}~\cite{ dwork2014algorithmic}. A differentially private algorithm ensures that its output cannot help much to distinguish whether an individual record is contained in the input dataset. In other words, DP hides the presence of individual records. The formal definition is as follows.

\begin{definition}
(Differential Privacy). Given a domain $\mathcal{D}$, any two datasets $D, D'\subseteq \mathcal{D}$ that differ in exactly one record are called neighboring datasets.  A randomized algorithm $\mathcal{M}: \mathcal{D} \rightarrow \mathcal{R}$ is $(\epsilon, \delta)$-differential private if for all neighboring datasets $D$ and $D'$ and all $T \subseteq \mathcal{R}$, 
\begin{align*}
    \Pr[\mathcal{M(D)}\subseteq T] \leq e^{\epsilon}\Pr[\mathcal{M(D')}\subseteq T] + \delta.
\end{align*}
\end{definition}


\section{Selective Differential Privacy}

Canonical DP notion treats all records as sensitive. Prior work has studied variants of DP notions, such as personalized DP~\cite{jorgensen2015conservative} and one-sided DP~\cite{doudalis2017one}, to exploit
different privacy levels between records. However, existing privacy notions do not allow different attributes in a given record to have different privacy levels, which could otherwise potentially enable additional utility gains, especially for NLP tasks where private attributes are sparse. 
Hence, we propose a new privacy notion--\emph{selective differential privacy}--to distinguish between private and non-private attributes inside one data point with a \emph{policy function} and protect sensitive part of one data point.


\begin{definition} (Policy Function). 
A policy function $F: \tau \rightarrow \{0, 1\}^{n_r}$ denotes which attributes of a record $r\in \tau$ are sensitive ($F(r)_i=0$) or non-sensitive ($F(r)_i=1$), where $n_r$ is the number of attributes in $r$. Note that $n_r$  depends on the record and is not a fixed number.  
\label{def:policy function}
\end{definition}

\wyshiwyshi{In practice, users have the freedom to define the policy function to encode specific privacy regulation and protect any sensitive attributes according to the applications. The protected sensitive attribute types are unlimited, can be entities (e.g., name, emails, etc), contextual (e.g., health-related information, speaking style, etc), and so on. For example, users can design a conservative policy function that protects selected complete sentences  if  necessary.   
The form of the policy function is also unlimited, and could be neural networks, regex, and so on. Please refer to Section~\ref{appendix: transformer} for contextual policy functions. }



In the case of language modeling, each record is a text sequence $\vx$, each attribute is a token $x_i$ in $x$ and $F(\vx)$ is a bit vector indicating which tokens contain private information. We define neighboring datasets under our new privacy notion as follows.  
\begin{definition} ($F$-Neighbors). $D, D'$ are two datasets and $F$ is a policy function.  $D'$ is a $F$-neighbor of $D$ if and only if $\exists r\in D$ s.t., 
$F(r)$ contains at least one private attribute, $\exists r'\in D'$ s.t., $F(r)$ and $F(r')$ differ by at least one private attribute, and $D'=D\backslash \{r\} \cup \{r'\}$. We denote $D'\in N_F(D)$.
\label{def:neighboring datasets}
\end{definition}

Under this definition, the dataset containing ``My ID is \textcolor{amber}{\hlapricot{123}}''  and the dataset containing ``My ID is \textcolor{teal}{\hllightcyan{456}}'' are neighbors; but the dataset with ``Hello there'' and the dataset with ``Hi there'' are not neighbors since they do not contain private information.

\begin{definition}
(Selective Differential Privacy). Given a policy function $F$,  
a randomized algorithm $\mathcal{M}: \mathcal{D} \rightarrow \mathcal{R}$ satisfies $(F, \epsilon, \delta)$-selective differential privacy if for $\forall D, D'\in N_F(D)$,  and  $\forall T \subseteq \mathcal{R}$, 
\begin{align*}
    \Pr[\mathcal{M}(D)\subseteq T] \leq e^{\epsilon}\Pr[\mathcal{M}(D')\subseteq T] + \delta.
\end{align*}
\end{definition}

Essentially, S-DP \wyshi{also} provides an indistinguishability property similar to canonical DP,
but only for the sensitive attributes \wyshi{in a record}. S-DP does not constrain the information leakage of nonsensitive attributes as long as the privacy of the sensitive attributes is preserved. \wyshiwyshi{Thus, S-DP protects privacy for sensitive attributes in the worst case (i.e., the attacker may have knowledge about everything except the targeted sensitive attribute.)
}


\subsection{Selective Privacy Mechanism}
With the new S-DP notion, the next step is to develop a corresponding privacy mechanism to train models that realize the new notion. Privacy mechanisms usually work by adding noise to the models to protect the data, such as Laplace mechanism (Laplace noise) and Gaussian mechanism (Gaussian noise). \citet{abadi2016deep} proposed DPSGD that adds Gaussian noise to the gradients and applies stochastic gradient descent (SGD) to train private deep learning models. In this work, we develop \textit{Selective-DPSGD}, shown in Figure~\ref{fig:model} and Algorithm~\ref{algorithm}, to train RNN-based language models that achieve S-DP. The basic idea is to first determine the private attributes with the policy function, then decide which model variables are related to the private attributes, and finally apply regular SGD on non-private variables, and  DPSGD \cite{abadi2016deep} on the private variables. \wyshiwyshi{We  choose RNNs because they are widely used in industry, e.g., \citet{ramaswamy2020training} discussed how to train private production language models with RNNs. 
}

We need to first decide the variables related to the private tokens. RNN uses a hidden state $\vh_i$ to encode the context, and outputs a distribution $\vp_i$ over a vocabulary set $V$, as shown in Equation~(\ref{eq:rnn}). If $x_i$ is private, then $\vh_i$, $\vp_{i}$, and $\mathcal{L}_{i}$  are all private; besides, to calculate $\mathcal{L}_{i-1}$, we need to access the ground truth next token $x_i$, so $\mathcal{L}_{i-1}$ is also private. The private variables are all in red in Figure~\ref{fig:model}. 
\begin{align}
\label{eq:rnn}
&\vh_i = RNN(\vh_{i-1}, x_i)\\
\label{eq:rnn1}
&\vp_i = p_{\theta}(V|\vx_{<i}) = Softmax(g(\vh_i))\\
&\mathcal{L}_{i} = -\log p_{\theta}(x_{i+1}|\vx_{<i+1})
\end{align}

\begin{figure}[h!]
    \centering
    \includegraphics[scale=0.42]{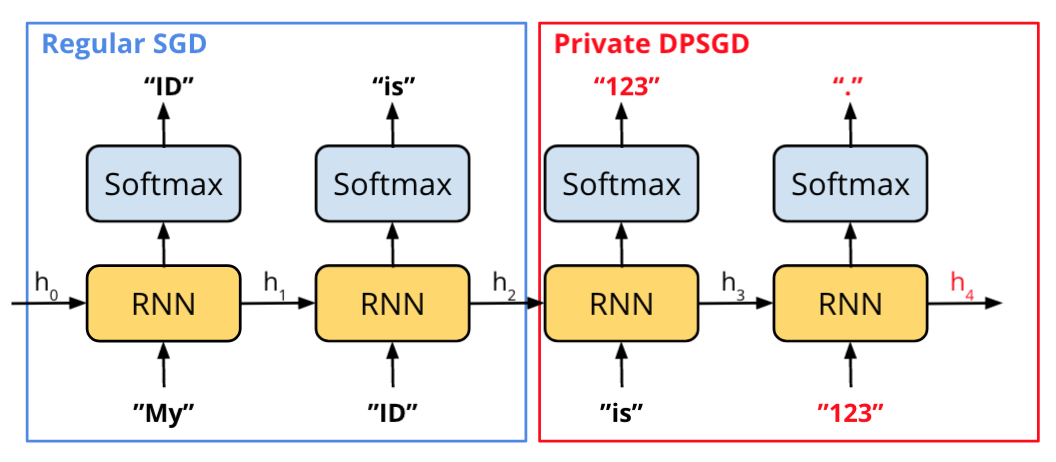}
    \caption{All private variables  are in red. We apply regular SGD on non-private variables and DPSGD on private variables in Selective-DPSGD. }
    \label{fig:model}
\end{figure}

Algorithm~\ref{algorithm} outlines the steps in \textit{Selective-DPSGD}. Given a dataset $D$, we apply a policy function $F$ to obtain a bit matrix $P = F(D)$ that indicates which tokens are private. 
At each step, we take a random batch $B$, and use $P$ to split $B$  into a sequence of non-private and private tuples $\{ (B_{np, i}, B_{p, i})\}$; then we apply SGD (regular update) on $B_{np,i}$  and DPSGD (private update) on $B_{p,i}$ alternatively, to update and protect privacy. Note that besides noise in the gradients, we also clip and add noise to the hidden state $\vh_i$ if it is private. The reason is that in RNN, if $x_i$ is private, $\vh_i$ also contains private information (as shown above), and is directly passed to the next regular update step and cannot be protected by the noise in the gradients. So it is important to add noise to protect the private information in $\vh_i$. 
Since DPSGD adds noise to the gradients, $\mathcal{L}$ and $\vp_{i}$ used to calculate the loss are protected by the noise in the gradients. In this way, all private variables are protected. 

\noindent\textbf{Privacy Guarantee.} In Section~\ref{appendix: proof}, we prove that the composition of the series of noise-adding operations ensures S-DP for  Selective-DPSGD. 

\begin{algorithm}[ht!]
\caption{Selective-DPSGD}
\small

\begin{algorithmic}[1]
\STATE \textbf{Input}: Dataset $D$ with $N$ examples, policy function $F$,  privacy bit matrix $P=F(D)$, max sequence length K,  loss function $\mathcal{L}(\theta)$.   \\ 
                        Parameters: learning rate $\eta$, noise multiplier $\sigma$, gradient norm bound $C$, group size $L$.
\FOR{t=1,2,...}
    \STATE Take a random batch $B$ of max sequence length K, with sampling probability $L/N$
    \STATE Using P, split $B$ into a sequence of non-private and private tuples $\{ (B_{np, i}, B_{p, i})\}$
    \STATE Initialize $\vh$ = $\Vec{0}$
    \FOR{i=1,2,...}
    \STATE \textbf{1) Regular update}
    \STATE $\mathcal{L}, \vh$ = Model($B_{np, i}$, $\vh$)
    \STATE $\theta \leftarrow \theta -\eta \nabla_{\theta}{\mathcal{L}(\theta)}$
    \STATE \textbf{2) Private update}
    \STATE $\mathcal{L}, \vh$ = Model($B_{p,i}$, $\vh$)
    \STATE \textbf{Calculate sample gradient}
    \STATE For each $x_j\in B_{p, i}$, compute $\vg(x_j) \leftarrow \nabla_{\theta}{\mathcal{L}(\theta, x_j)}$ 
    \STATE \textbf{Clip gradient}
    \STATE $\vg(x_j) \leftarrow \vg(x_j)/\max(1, \frac{\Vert \vg \Vert_2}{C})$ 
    \STATE \textbf{Add Noise}
    \STATE $\vg(x_j)\leftarrow\frac{1}{|B_{p, i}|}(\sum_j \vg(x_j) + \sigma C \cdot\mathcal{N}(0,\textbf{I}))$
    \STATE \textbf{Descent}
    \STATE $\theta \leftarrow \theta -\eta \nabla_{\theta}{\mathcal{L}(\theta)}$
    \STATE \textbf{Clip hidden states}
    \STATE $\vh (x_j) \leftarrow \vh(x_j)/\max(1, \frac{\Vert \vh \Vert_2}{C})$ 
    \STATE \textbf{Add Noise}
    \STATE $\vh(x_j)\leftarrow \vh(x_j) + \sigma C \cdot \mathcal{N}(0, \textbf{I})$ 
    
\ENDFOR
\ENDFOR
\end{algorithmic}
\label{algorithm}
\end{algorithm}

\section{Experiments}
We conduct our experiments on two datasets: 1) a traditional text corpus for language modeling, and 2) a  dialog dataset for a more concrete application of dialog systems. Below are the dataset details.  

\noindent \textbf{WikiText-2.} To minimize real-world harm, we choose the already-public WikiText-2 \cite{merity2016pointer}. It contains articles from Wikipedia with potentially sensitive information, and is a classical dataset for language modeling. For simplicity, we treat all the digits as privacy information. So the policy function $F$ is a digit detector: if the token is a digit, $F$ will output 0, otherwise, 1.

\noindent\textbf{\textsc{CustomerSim}.} With the emergence of virtual assistants, more and more private information is being exchanged during daily interactions. So  we also apply S-DP to build dialog systems in the customer service domain. Using real dialog transcripts may lead to real-world harm, so we simulate a dialog dataset, \textsc{CustomerSim}, with synthetic user information. The dialogs are simulated with fixed agendas and template utterances \cite{zhao2018zero}. 
We treat user name, address, phone number, order, and tracking number as sensitive information, and use regex to build a policy function to detect them.  Table~\ref{tb:dialog examples} shows one example dialog.

\wyshiwyshi{Note that although we use digits and names as running examples for sensitive information, S-DP can protect any sensitive attributes specified by the policy function. Building better policy functions  is orthogonal to S-DP, and thus beyond the scope of this paper. Any improvements on policy functions are compatible with S-DP to achieve better results.}

\begin{table}[ht]
\small
\centering
\begin{adjustbox}{width=.95\columnwidth}
\centering
\begin{tabular}{p{0.1\linewidth} | p{0.85\linewidth}}

\toprule

\multicolumn{2}{c}{

\textbf{\textsc{CustomerSim}} 

}\\

\midrule
\rowcolor{white} \textbf{Role} & \bf Utterance\\
\midrule
\rowcolor{mygray} SYS   &  Hello, I am the customer support bot. What can I do for you? \\\rowcolor{white}
USR   & Hello robot. Could you please help me track my package?\\\rowcolor{mygray}
SYS   &  Please provide your full name.\\\rowcolor{white}
USR   & Sure,  \textcolor{amber}{\hlapricot{Betty Sims}}.\\\rowcolor{mygray}
SYS   &  Could you please confirm your shipping address?\\\rowcolor{white}
USR   & Yea sure, \textcolor{amber}{\hlapricot{2241 Fitzgerald Viaduct Brownview, OK 28304}}.\\\rowcolor{mygray}
SYS   & Track your order using your tracking number,  \textcolor{amber}{\hlapricot{FH6F6GMMF4}}. Are you happy about my answer? \\\rowcolor{white}
USR   &  That's it.\\
\bottomrule

\end{tabular}
\end{adjustbox}
\caption{An example dialog in  \textsc{CustomerSim}.}
\label{tb:dialog examples}
\end{table}

\begin{table*}[!htbp]
    \centering
    \resizebox{\textwidth}{!}{
    \begin{tabular}{l|cccccc|cccccc}
    \midrule
    &\multicolumn{6}{c|}{\textbf{WikiText-2}} & \multicolumn{6}{c}{\textbf{\textsc{CustomerSim}}}\\
\midrule
 \textbf{Model}
 & \textbf{PPL} 
 & \textbf{AccT1} 
 & \textbf{$\sigma$}   &  \textbf{$C$}   &  \textbf{$\epsilon$} &  \textbf{$\delta$} 
 & \textbf{PPL} 
 & \textbf{AccT1}
 & \textbf{$\sigma$}   &  \textbf{$C$}   &  \textbf{$\epsilon$} &  \textbf{$\delta$} 
 \\ 
 \midrule

LSTM, No-DP 
&60.98 $\pm$ 0.48 & 0.36 $\pm$ 0.00 
& - &- &- &- 
&3.06 $\pm$ 0.01& 0.75 $\pm$ 0.00
&- &- &- &- 
\\

\midrule
DPSGD 
& 305.86 $\pm$ 3.00 & 0.27 $\pm$ 0.00 
& 0.50   & 0.10  & 4.89    & 8e-5        
& 11.82 $\pm$ 0.76 & 0.70 $\pm$ 0.01 
& 0.60   & 0.01  & 2.51    & 8e-5        

  \\ 
\midrule

S-DPSGD 
& 160.04 $\pm$ 4.86 &  0.31 $\pm$ 0.00   
& 0.50   &   1e-3 & 4.91    & 8e-5       


& 10.42 $\pm$ 0.91  & 0.69 $\pm$ 0.01 
& 0.70&  5e-3 & 2.74    & 8e-5        
\\
\midrule
\end{tabular}
}
    \caption{Model utility and privacy guarantee on WikiText-2 (left) and \textsc{CustomerSim} (right). \textbf{PPL}: Perplexity on the test set. \textbf{AccT1}: Top-1 next word prediction accuracy. \textbf{$\sigma$}: Noise multiplier in the Gaussian noise. \textbf{$C$}: Clipping threshold. \textbf{$\epsilon, \delta$}: Privacy guarantee in ($\epsilon$, $\delta$)-privacy.
    }
    \label{tab:model utility:wiki}
\end{table*}

\noindent \textbf{Model training details.} We use one-layer LSTMs with an embedding size of 200 and a hidden size of 200, 
and a BPE tokenizer \cite{sennrich2015neural}  to avoid information leakage from the tokenizer: with BPE, a secret ``1234'' will be split into multiple tokens, e.g., ``12'' and ``34'', while traditional tokenizer will release ``1234'' in the dictionary. All private states (hidden, cell states) in the LSTMs are protected.


\noindent \textbf{Baselines.} We have two baselines, one without DP (``No-DP''), and the other trained with DPSGD (``DPSGD''). We refer to our models trained with S-DPSGD as ``S-DPSGD''.  ``No-DP''  is simply an LSTM optimized with a regular SGD and a starting learning rate (lr) of 20. \wyshi{The learning rate was annealed and decreased as training proceeded.} 
``DPSGD'' is optimized with DPSGD and a starting learning rate of 0.05.  All the models are trained five times to reduce randomness, and the parameters are tuned on the validation set. \wyshiwyshi{We compare with DPSGD because it's the backbone of most of existing DP learning algorithms. Existing modifications of DPSGD are mainly focused on optimization algorithms and objectives, thus are compatible with our work that tailors the privacy notions to realistic privacy needs in the NLP context.} 

\subsection{Evaluation}

We evaluate both the language models' utilities and privacy protection levels. 
We use perplexity (PPL) and the top-1 next word prediction accuracy (AccT1) to measure model utility. To measure privacy protection levels, besides reporting the theoretical privacy budget $\epsilon$ and $\delta$, we also perform various practical attacks on the trained models and report how successful the attacks are against different techniques. We compare the performance of our proposed privacy-preserving learning technique and the baselines in terms of the privacy-utility trade-off. Specifically, we compare the utility between different techniques at a \emph{given} privacy protection level, or vice versa.

\begin{figure*}[!htbp]
    \centering
    \includegraphics[height=4.5cm, width=\textwidth]{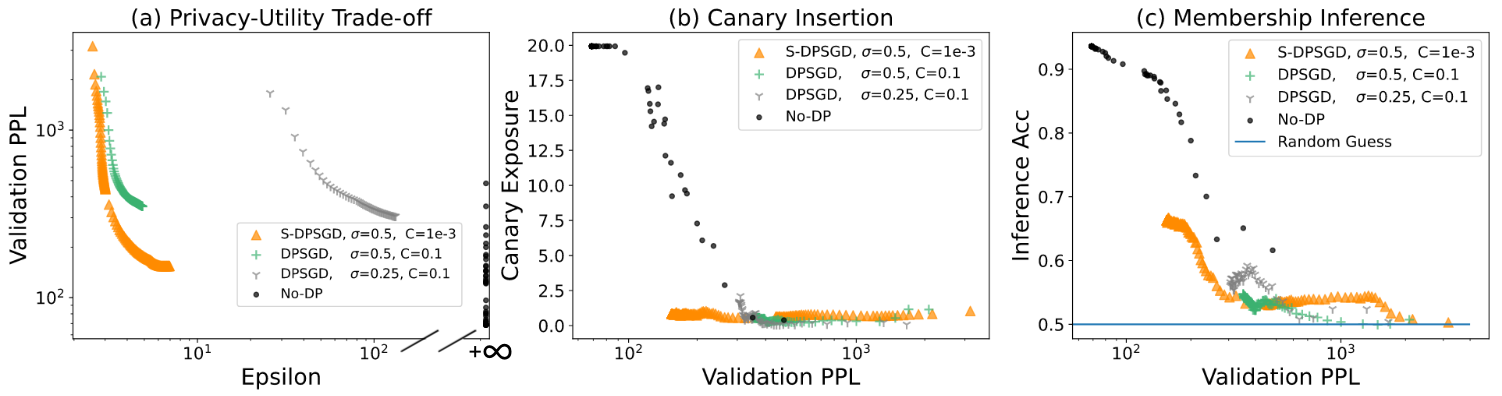}
    \caption{Privacy-utility trade-off, canary insertion attack and membership inference attack on WikiText-2.}
    \label{fig:three together,wiki}
\end{figure*}

\subsubsection{Attack Details}
We perform two types of attacks: 1) canary insertion and 2) membership inference. 

\noindent\textbf{Canary insertion} is proposed by \citet{carlini2019secret}. It first inserts random sequences called canaries into the training dataset, then trains the model, and finally calculates the following \textit{exposure} for the inserted canaries to measure if the model will unintentionally memorize these canaries. Canaries are of a specific format, e.g., s=``The number is \dottedcircle\dottedcircle\dottedcircle\dottedcircle\dottedcircle\dottedcircle'', where \dottedcircle\space are filled with random values from a randomness space $\mathcal{R}$ such as  a space of digits from 0 to 9. To obtain the \textit{exposure}, we enumerate all possible sequences in the specified form, and calculate the negative log-rank with an additional constant, as shown below. Lower exposure indicates the model is more private.  In our setting, we insert the secret ``My ID is \textcolor{amber}{\hlapricot{341752}}'' into the training data for 10 times to make the differences between models more salient.

\begin{definition}
Given a canary $s[r]$, a model with parameters $\theta$, and the randomness space $\mathcal{R}$, the \textbf{exposure} of $s[r]$ is
\begin{align*}
\vspace{-0.5em}
    \textbf{exposure}_\theta = \log_2|\mathcal{R}|-\log_2\textbf{rank}_\theta(s[r])
    \vspace{-0.5em}
\end{align*}
\end{definition} 

\noindent\textbf{Membership Inference} is a widely used attack method that identifies if a given sample is a member of the training dataset. Lower inference accuracy means that it is harder to infer a member from the model and thus the model is safer.  \citet{carlini2020extracting} proposed an advance membership inference attack for language models. The basic idea is to calculate the given samples' perplexities under the model, rank them and choose the ones with the lowest perplexities (highest likelihood by the model). In our experiments, we randomly select 500 protected secrets from the training set, and randomly generate 500 samples of similar format, to form a dataset for the membership inference attack, so a random guess would give us an accuracy of 50\%. For  WikiText-2, the secrets are digit sequences; for \textsc{CustomerSim},  customer names are the secrets. 



\subsection{WikiText-2 Results}

\begin{figure*}
    \centering
    \includegraphics[height=4.5cm, width=\textwidth]{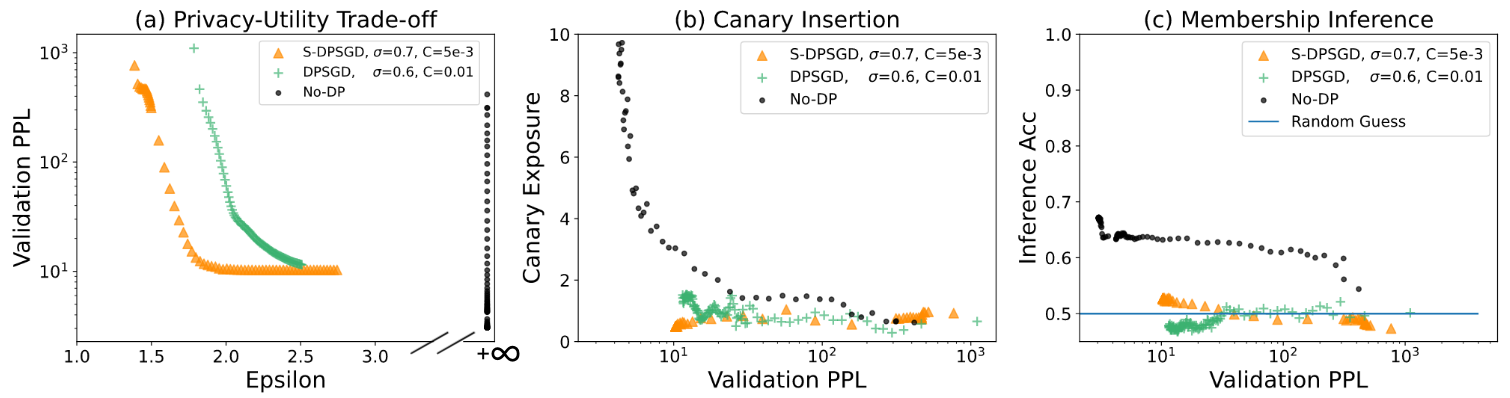}
    \caption{Privacy-utility trade-off, canary insertion attack and membership inference attack on \textsc{CustomerSim}. }
    \label{fig:model utility: dialog}
\end{figure*}

\noindent\textbf{Model utility and privacy guarantee.} The left part of Table~\ref{tab:model utility:wiki} shows different models' utilities and privacy guarantees on WikiText-2 and Figure~\ref{fig:three together,wiki}(a) shows the privacy-utility trade-off, where utility is represented by the validation perplexity (lower PPL=higher utility) and the privacy budget is represented by the epsilon (smaller $\epsilon$=more privacy). Although the definitions of $\epsilon$ in DPSGD and S-DPSGD are different, the $\epsilon$ in both cases provides a \emph{tight} theoretical upper bound on how much the sensitive attributes are leaked through the trained models. Hence, the $\epsilon$ associated with two algorithms are comparable.  Given a fixed privacy budget $\epsilon$, we want to achieve a higher model utility. Because the privacy budget $\epsilon=+\infty$ for the No-DP model, ``No-DP'' is represented by the vertical black line on the far right in Figure~\ref{fig:three together,wiki}(a), and it achieves the best average PPL of 60.98 on the test set. The orange line is our S-DPSGD model, and it achieves the second-best average test PPL of 160.04 with $\epsilon=4.91$.  With a similar $\epsilon=4.89$, DPSGD has the worst average test PPL of 305.86, much worse than S-DP because canonical DP notion protects the whole data and is over-pessimistic (see Section~\ref{appendix: param search} for models with different parameters).  \wyshi{The gray line is for DPSGD with a smaller $\sigma=0.25$, a convergent PPL of 266.6 and a final $\epsilon=132.73$. Compared to DPSGD with $\sigma=0.25$,  it achieves a better PPL but with a much higher cost of privacy leakage (larger $\epsilon$). But with S-DPSGD, we can also achieve lower PPL without hurting privacy.}

\noindent\textbf{Attack results.} Figure~\ref{fig:three together,wiki}(b) and~\ref{fig:three together,wiki}(c)  show the canary insertion attack and membership inference attack results on WikiText-2. The x-axis is the models' utilities measured by validation PPL, and the y-axis is the exposure and membership inference accuracy indicating the success level of the attacks. Lower exposure and lower accuracy indicate a safer model.  \wyshi{We want to see at a given robustness level to the attacks, which models can achieve lower perplexity, i.e., higher utility.}

For canary insertion attack, although ``No-DP'' achieves lower perplexity,  its exposure can go up to 20, indicating that the inserted canary could be easily revealed by the attackers. If we compare No-DP with S-DPSGD with similar utilities, S-DPSGD is always below No-DP, meaning S-DPSGD achieves much smaller exposure, and hence a safer model with similar utility. Comparing DPSGD and S-DPSGD, we find that S-DPSGD achieves much better model utility at a given exposure. 

For membership inference attack, we draw a horizontal line of 0.5 to show the random guess performance.  
Again,  S-DPSGD is always below No-DP, showing that with similar utilities,  models trained with S-DPSGD are safer than No-DP under the membership inference attack. 
As mentioned earlier, DPSGD with $\sigma=0.5$ (green)  and S-DPSGD (orange) have similar privacy budget ($\epsilon$=4.89 and 4.91 respectively). 
Comparing these two, we see that given a similar privacy budget, S-DPSGD  converges to a much lower perplexity while remaining safe to the attack and thus achieves a wider range for the privacy-utility trade-off tuning.  
We also observe that for ``No-DP'', the inference accuracy can go up to 90\%, suggesting that language models without DP protection are vulnerable to attacks.

\subsection{\textsc{CustomerSim} Results}

This section shows the results on \textsc{CustomerSim}. 

\noindent\textbf{Model utility and privacy guarantee.} 
The right part of Table~\ref{tab:model utility:wiki} and Figure~\ref{fig:model utility: dialog}(a) show the privacy-utility trade-off for \textsc{CustomerSim}. Because the dialogs are simulated with templates and relatively simple, the perplexity can be as low as 3.06 for No-DP.  S-DPSGD still achieves better perplexity than DPSGD (10.42 vs. 11.82) with similar $\epsilon$, but the gap is smaller compared to WikiText-2 because there are more sensitive tokens in \textsc{CustomerSim} (18.3\%) than WikiText-2 (2.8\%), and the advantage of protecting selective tokens only is not as big. 

\noindent\textbf{Attack results.} 
Figure~\ref{fig:model utility: dialog}(b) and~\ref{fig:model utility: dialog}(c) show the results of the canary insertion  and membership inference attack on \textsc{CustomerSim} respectively. 

For canary insertion,  we observe that given the same utility, S-DPSGD achieves lower exposure than No-DP and DPSGD. Although the improvement seems small on absolute values, we should note that exposure is on log-scale, so the improvement on relative rank is also on log-scale (e.g., for exposure=1.2, the rank of the canary is 429881; for exposure=0.2, the rank is 858215, roughly twice). 


The membership inference accuracy is only a little better than random guess probability (60+\%), so it is not successful on \textsc{CustomerSim}. 
One potential reason could be that customer names only appear once in each dialog and can be very similar to each other (e.g., Emily, Emilya, and Emilyah). We leave it as future work to develop better membership inference attacks towards similar secrets. Under the failed attack, S-DPSGD is still better than No-DP. It is not feasible to compare S-DPSGD and DPSGD as both are close to random guess.  

\subsection{Data Anonymization and Selective-DP}

\wyshiwyshi{Data anonymization (or data de-identification) has been widely used to protect data privacy, where the sensitive tokens are masked with special tokens such as ``<num>''. 
Both data anonymization and S-DP \wyshi{target protection towards sensitive attributes}, and rely on a policy function to detect the private information. 
However, they are different in that data anonymization masks the sensitive attributes completely so nothing can be learned from them, while S-DP noises the private portion and provides a tunable way to adjust the privacy-utility trade-off, evidenced by experiments in Section~\ref{sec: privacy-utility trade-off}.}

\wyshiwyshi{One common problem for both methods is that the policy function is not guaranteed to  be perfect and can miss to detect some sensitive information. So we also compare their performance when the policy function is imperfect in Section~\ref{sec:imperfect policy function}. }

\subsubsection{S-DP achieves better utility}
\label{sec: privacy-utility trade-off}
\begin{figure}[!htbp]
    \centering
    \includegraphics[scale=0.4]{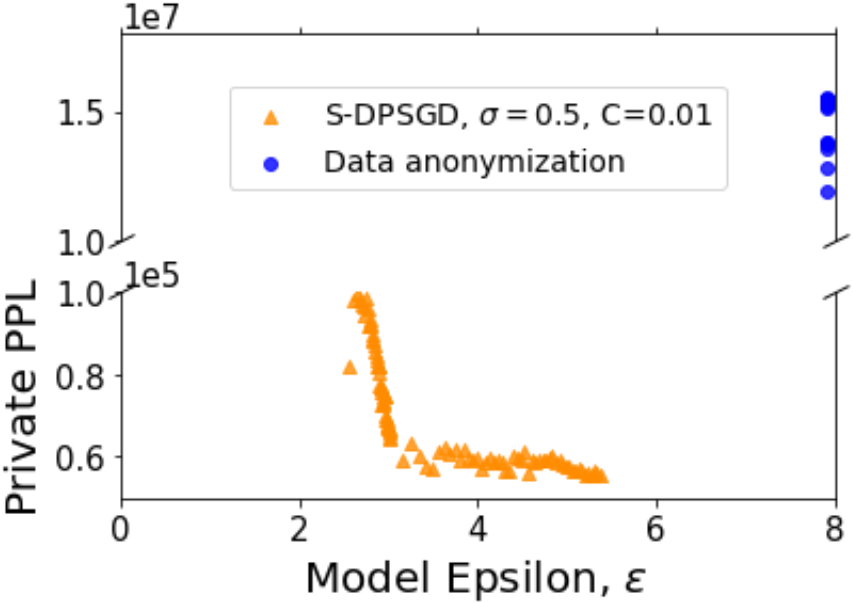}
    \caption{Perplexity on private tokens over $\epsilon$ for data anonymization and S-DPSGD.  }
    \label{fig:data anonymization}
\end{figure}


We mask all the detected digits by ``<num>'' to train a data anonymization baseline on WikiText-2, calculate the perplexity on the private tokens, and present the result in Figure~\ref{fig:data anonymization}. 
The x-axis is the privacy budget $\epsilon$, and the y-axis is the private-token perplexity.  For data anonymization, the dots are all on the far right because $\epsilon$$=$$+\infty$. The first observation is that since data anonymization simply masks the sensitive tokens, it fails to learn anything about them, resulting in a much worse PPL on private tokens. This makes S-DP a good alternative to data anonymization because S-DP can still learn certain patterns from the noised sensitive tokens. Also, Figure~\ref{fig:data anonymization} shows that for S-DP, there is a trade-off between $\epsilon$ (privacy) and private-token PPL (utility), so we could tune $\epsilon$ to achieve a better private-token PPL, while for data anonymization, there is no way to tune the parameters for better model utilities. 
\wyshi{More concretely, 
with proper $\epsilon$ and $\delta$, S-DP might learn the structure of sensitive attributes (e.g., XXX-XX-XXXX for SSN) without knowing the exact value, or even learn the distribution of values for each digit, and such knowledge could be useful for data analytics. But for data anonymization, the model either sees the digit or doesn’t see it, so there is no knob to tune the privacy-utility trade-off.}

\subsubsection{Imperfect Policy Function
}
\label{sec:imperfect policy function}

\begin{figure}[!htb]
    \centering
    \includegraphics[scale=0.4]{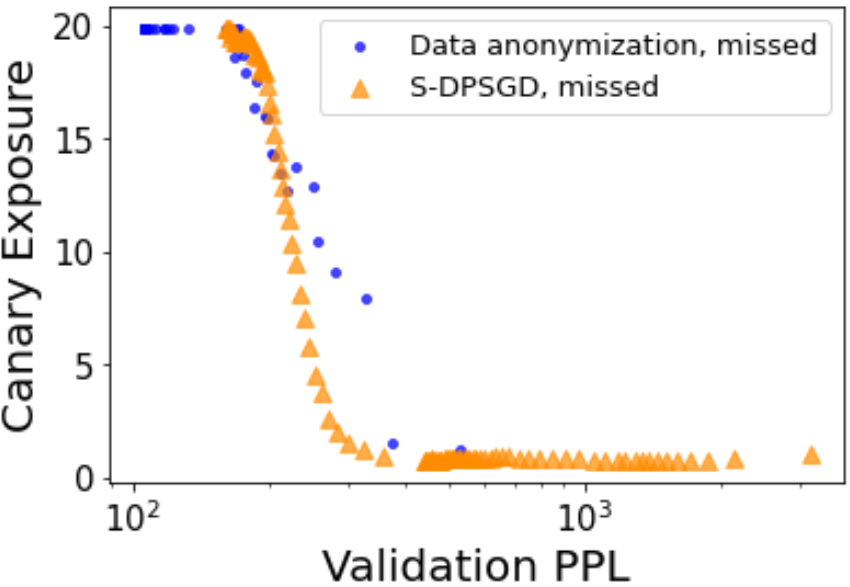}
    \caption{Canary insertion attack for data anonymization and S-DPSGD when missing the canary.}
    \label{fig:missed, canary}
\end{figure}

Now we discuss the performance of both methods when the policy function is imperfect. We still use WikiText-2 with the secret ``My ID is \textcolor{amber}{\hlapricot{341752}}''. The policy function  fails to detect ``341752'' as a secret. For data anonymization, all the detected digits are masked by ``<num>''; for S-DP, we apply S-DPSGD to noise the detected digits.



Figure~\ref{fig:missed, canary} shows the exposure of the missed secret \textcolor{amber}{\hlapricot{341752}}. When the perplexity gets lower, the model becomes better at remembering the details of the training data, so the exposure of both models becomes high. But when the perplexity is around the middle area, S-DP has lower exposure than data anonymization, meaning it's  safer to attacks.

\wyshiwyshi{Note that the risk with imperfect policy functions is common to many privacy-preserving techniques, and how to build better policy functions  is orthogonal to this work on S-DP, and thus beyond the scope of this paper. Improvements on policy functions (better sensitive information detection) are compatible with  S-DP and can be used to further improve the results. We are also actively working on this topic in parallel to the S-DP work. 
}

\section{Preliminary Results on Contextual Policy Function and Transformer}
\label{appendix: transformer}

\begin{table}[!htbp]
\centering
\begin{adjustbox}{width=.95\linewidth}
\begin{tabular}{l|p{0.5\linewidth}|l|ll}
\toprule
\textbf{Model}        & \textbf{Policy function}                                               & \textbf{Portion}                  & $\epsilon$ & \textbf{PPL}   \\
\midrule
GPT2 + no DP & -                                                             & -                        & -                       & 20.47 \\
DPSGD        & -                                                             & -                        & 2.58                    & 27.05 \\
\midrule
Redacted     & \multirow{2}{*}{\shortstack[l]{Noncontextual:\\ All entities}}                                 & \multirow{2}{*}{16.40\%} & -                       & 24.30 \\
S-DP          &                                                               &                          & 2.58                    & 22.56 \\
\midrule
Redacted     & \multirow{2}{*}{\shortstack[l]{Contextual: entities, sub, \\ obj, propon, pron}} & \multirow{2}{*}{34.80\%} & -                       & 38.66 \\
S-DP          &                                                               &                          & 2.48                    & 25.61\\
\bottomrule
\end{tabular}
\end{adjustbox}
\caption{Preliminary results on different policy functions and Transformers. }
\label{tab:transformers}
\end{table}

In this section, we present the preliminary results on different policy functions and large Transformer models \cite{vaswani2017attention} on WikiText-2. 

We design two policy functions: one is non-contextual and protects all the 18 named entities detected by spacy \cite{spacy2} such as person, date, locations, etc (16.4\% tokens)\footnote{The full list of entities is available here \url{https://spacy.io/usage/linguistic-features\#named-entities}}; the other one is contextual that protects all the entities plus subjects, objects, proper noun and pronouns of all the sentences (34.8\% tokens). We use these two policy functions to redact the WikiText-2 $D$ and obtain a redacted version $D'$. We first fine-tune a GPT2-small model \cite{radford2019language} on $D'$ (denoted as ``redacted model''), and then further fine-tune this redacted model on the original $D$ \cite{shi2022just}. The results are in Table~\ref{tab:transformers}. The redacted models trained on the redacted data $D'$ achieve 24.30 and 38.66 in perplexity for the two  policy functions respectively. If we fine-tune these two models on the original private data $D$ with DPSGD, we can further improve the perplexity to 22.56 and 25.61, while the state-of-the-art DP language models only achieve 27.05 with similar privacy budget. These  results show that our S-DP notion is promising in boosting utility of privacy-preserving language models even if one-third of the tokens are considered sensitive. 

\section{Conclusions}
To conclude, we develop a new privacy notion, \emph{selective differential privacy} (S-DP), to improve model utility while providing rigorous privacy guarantees for the sensitive portion of the data. We also develop a privacy mechanism, Selective-DPSGD, to achieve the new S-DP notion for RNNs. We experiment with WikiText-2 and a synthetic customer service dialog dataset.  Results on both tasks show that models trained with S-DP achieve better utilities than traditional DP, and are more robust under various attacks than models without DP protection. 
With S-DP, we march one step closer towards safer and better language models and hope to inspire more related research. 
Moreover, S-DP could be applied to domains beyond NLP where only partial data require protection, such as image recognition. Please see Section~\ref{appendix: ethical} for ethical consideration.




\bibliography{anthology,custom}
\bibliographystyle{acl_natbib}
\clearpage
\appendix

\section{Appendix}
\label{sec:appendix}

\subsection{Ethical Consideration}
\label{appendix: ethical}
\textbf{Data Usage.} To minimize real-world harm, we choose WikiText-2 since it is already public and widely used, and synthesize the dialogs as well as the personal information in the \textsc{CustomerSim} datasets in our experiments. For future research, we plan to continue using public or synthetic datasets to prevent real-world data leakage.

\noindent\textbf{Application.} Our work addresses the problem of data privacy protection and can be applied in different applications. The attacks used in our study are well-known standard attacks tailored for our specific tasks, so it's hard to generalize and misuse them to attack other language models. We will release the code so that people can have access to the various algorithms and protect their own data. 

\subsection{Limitations}
There are many spaces for improvements for this S-DP work. For instance, when the policy function fails to detect sensitive attributes, their privacy may not be guaranteed, and therefore, we plan to develop better policy functions and employ privacy amplification \cite{balle2018privacy}. Also,  
besides explicit private information like names, we plan to protect sensitive context such as ``I have two kids'', as these can often happen causally in dialogs but still reveal personal status. In Section~\ref{appendix: transformer} we show some preliminary results on protecting contexts and plan to further refine the contextual policy functions. 

\subsection{Privacy Analysis}
\label{appendix: proof}
We analyze the private guarantees of Selective-DPSGD in this section. 

For any given dataset $D$, let $D_{i,j}$ denote the $j$th attribute of the $i$-th record. We abstract the gradient update and hidden state into a query function $f(x,w)$ which takes training data $x$ and auxiliary information $w$ as input. We introduce $w$ as an additional input to $f$  to model the dependence of the gradient update and hidden state on the model parameters at the previous rounds. We define the following two types of queries on the dataset.
\begin{itemize}
    \item Type-1 query: the input $x$ to $f$ consists of only private attributes with respect to the policy function $F$
    \item Type-2 query: the input $x$ to $f$ consists of only non-private attributes with respect to the policy function $F$
\end{itemize}

Since S-DP only hides the presence of private attributes, type-2 query does not incur privacy loss. 




The following theorem shows that if a type-1 query has the property that its output is bounded, then for arbitrary auxiliary input, adding Gaussian noise into the query can provide DP. The reason why we consider such queries is that clipped gradient and hidden state can be modeled as such queries. The reason for which we want to analyze DP guarantees under arbitrary auxiliary inputs is that at any given round, the model parameters resulting from previous rounds could be arbitrarily different. This is because the non-sensitive part of two $F$-neighboring datasets could be arbitrarily different.

\begin{theorem}
\label{thm:bound}
Assume that $\max_{x,w}\|g(x,w)\|\leq C$. Then, for any arbitrary $w$, adding Gaussian noise $\Delta =\mathcal{N}(0,\sigma^2)$ proportional to $C$ into $g$ can ensure $(\epsilon,\delta)$-DP where $\epsilon,\delta$ depends on $C$ and $\sigma$. More formally, for all neighboring datasets $x$ and $x'$ and all $w,w'$,
\begin{align}
    \frac{P[g(x,w)+\Delta=r]}{P[g(x',w')+\Delta=r]}\leq e^\epsilon \quad \text{w.p. } 1-\delta 
\end{align}
\end{theorem}
The proof follows directly from the classic proof for DP guarantees of the Gaussian mechanism~\cite{mironov2017renyi, dwork2014algorithmic} by noticing the sensitivity of $f$ is bounded by $C$.

The regular updates in Algorithm~\ref{algorithm} take as input non-private data $B_{np,i}$. Hence, they are type-2 queries and do not incur extra privacy loss. The private updates in Algorithm~\ref{algorithm} (i.e., gradient and hidden states) depend on private attributes and model parameters from previous rounds, and thus belong to the type-1 query. Moreover, they satisfy the bounded norm assumption in Theorem~\ref{thm:bound}. We call the resulting query of adding Gaussian noise into a type-1 query with bounded norm property a \emph{noisy type-1 query}.
Overall, Algorithm~\ref{algorithm} is essentially the composition of multiple type-1 and noisy type-2 queries. In the following, we will present a general result for the privacy guarantees resulting from the composition. 

\begin{theorem}
\label{thm:dp}
Let $f$ be the composition of $k$ queries: $f_1,\ldots, f_k$, which are either noisy type-1 or type-2 queries.
Given a policy function $F$, let $\mathbf{f}_p$ denote the set of noisy type-1 queries.
Let $\mathbf{f}_{np}$ denote the set of type-2 queries. Then, if $\mathbf{f}_p$ is $(\epsilon,\delta)$-DP, $f$ is $(F,\epsilon,\delta)$-S-DP.
\end{theorem}

\begin{proof}
Consider two selective $F$-neighboring datasets $x$ and $x'$. Let $x_i$ and $x_i'$ be the subset of data utilized by $f_i$. $x_i$ contains only private attributes when $f_i$ is type-1 and contains only non-private attributes when $f_i$ is noisy type-2. By the neighboring relation between $x$ and $x'$, $x_i$ and $x_i'$ are also selective $F$-neighbors when $f_i$ is a type-1 query. In addition to the dataset, $f_i$ takes the output of all the previous queries. For a fixed outcome $(y_1,\ldots, y_k)$ of $f$, we have
\begin{align}
    &\frac{P[f_1(x_1,w_1)=y_1,\ldots,f_k(x_k,w_k)=y_k]}{P[f_1(x_1',w_1')=y_1,\ldots,f_k(x_k',w_k')=y_k]}\\
   = &\prod_{f_i \in \mathbf{f}_p} \frac{f_i (x_i,w_i)}{f_i (x_i',w_i')}\\
   & \leq e^\epsilon\quad  \text{w.p. } 1-\delta 
\end{align}
as desired. 

The equality in the second line is due to the fact that $\mathbf{f}_{np}$ does not incur privacy loss and the independence of randomness of each query given the output of all the previous queries. The inequality in the third line is due to the assumption that $\mathbf{f}_p$ is $(\epsilon,\delta)$-DP with arbitrary auxiliary input.
\end{proof}



Instantiating the type-1 and type-2 queries in Theorem~\ref{thm:dp} with the regular and private updates defined in Algorithm~\ref{algorithm} yields the privacy guarantees for Algorithm~\ref{algorithm}. Theorem~\ref{thm:dp} provides a convenient way of calculating the S-DP guarantees for Algorithm~\ref{algorithm}: one can apply off-the-shelf privacy composition tools to calculate the overall DP guarantees for all the private updates and then the entire algorithm satisfies S-DP with the same values of $\epsilon$ and $\delta$. Specifically, in this paper, we leverage moment accountant~\cite{abadi2016deep} to calculate the DP guarantees for the composition of all privacy queries.

\subsection{Models with Different Parameters}
\label{appendix: param search}
We plot the performances of models with different parameters in Figure~\ref{fig:model utility: wiki, param}. We find that fixing the noise multiplier $\sigma=0.5$, the clipping threshold $C$ has a big impact on the performance, and it cannot be too big (0.01) or too small (5e-6); if we fix $C$ and change $\sigma$ from 0.5 to 0.1, the perplexity will be lower but at a huge cost of the privacy budget $\sigma$. As expected, there is always a trade-off between the utility and the privacy spent, so we choose a balancing point with a reasonable utility and privacy guarantee ($\sigma=0.5$ and $C=$1e-3) for the main experiments.

\begin{figure}[!htb]
    \centering
    \includegraphics[scale=0.4]{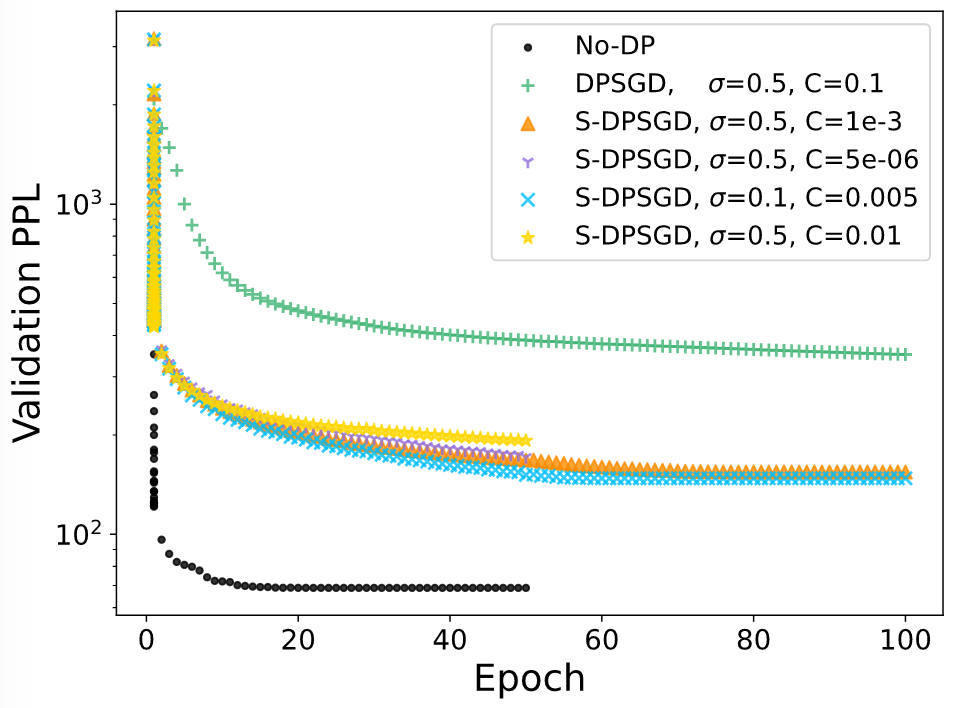}
    \caption{Validation perplexity over epochs on WikiText-2 for models with different parameters. }
    \label{fig:model utility: wiki, param}
\end{figure}

\subsection{Membership Inference on \textsc{CustomerSim}}
\begin{figure}[!htb]
    \centering
    \includegraphics[scale=0.5]{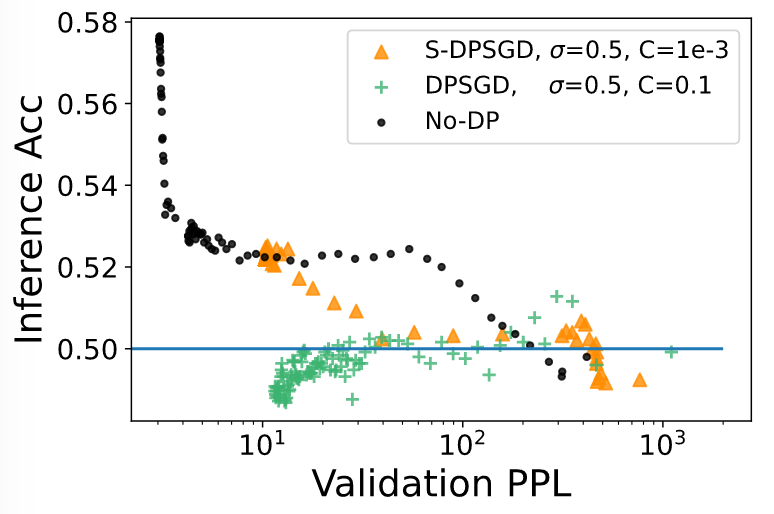}
    \caption{Original membership inference results on \textsc{CustomerSim}. The best inference accuracy is around 58\%.}
    \label{fig:original mem inf, dialog}
\end{figure}

The original membership inference attack doesn't achieve good results on \textsc{CustomerSim}. Figure~\ref{fig:original mem inf, dialog} shows the original membership inference result on \textsc{CustomerSim}. The best inference accuracy is around 58\%. So we employ a more advanced version, where we first perform the attack on 1000 names, and then pick the best-predicted and worst-predicted names to form a new subset of 300 names to perform the attack again. But even for the advanced version, the inference accuracy is only a little better than random guess probability (60+\%), so this attack is not successful on \textsc{CustomerSim}.

\end{document}